\documentclass{article}
\usepackage{algorithm}
\usepackage{pifont}
\usepackage{booktabs}
\usepackage{multirow}
\usepackage{graphicx}
\usepackage{algorithmic}

\usepackage{amssymb,amsthm,bbm,amsmath}
\usepackage[numbers]{natbib}

\usepackage{times}

\newtheorem{thm}{Theorem}
\newtheorem{defn}{Definition}
\newtheorem{lemma}{Lemma}

\newtheorem{corr}{Corollary}

	\newcommand{\nf}{\mathcal{F}}	
	\newcommand{\reals}{\mathbb{R}}
	
	\newcommand{\nr}{\mathbb{R}}
	
	\newcommand{\np}{\mathbb{P}}
	\newcommand{\nq}{\mathbb{Q}}
	\newcommand{\nE}{\mathbb{E}}
	\newcommand{\lamoptim}{\lambda_\infty^*}
    \newcommand{\xoptim}{y_\infty^*}	
    \newcommand{\Strategy}{\mathcal{S}}
    \newcommand{\strategy}{S}
    \newcommand{\cX}{\mathcal{X}} 
    \newcommand{\cY}{\mathcal{Y}} 
    \newcommand{\obsS}{\mathcal{X}} 
    \newcommand{\choS}{\mathcal{Y}} 
    \newcommand{\obs}{X} 
    \newcommand{\obsr}{x} 
    \newcommand{\cho}{y} 
    \newcommand{\obsF}{X} 
    \newcommand{\ml}{u} 
    \newcommand{\cl}{c} 
    \newcommand{\optimal}{\mathcal{V}^*} 
    \newcommand{\stri}{w}

\newcommand{\eqdef}{\triangleq}

\title{Multi-Objective Non-parametric Sequential Prediction}
\author{Guy Uziel  \\ 
Technion -- Israel Institute of Technology  \\
Ran El-Yaniv  \\ 
Technion -- Israel Institute of Technology  \\
}

\begin{document}

\maketitle

\begin{abstract} 
Online-learning research has mainly been focusing on minimizing one objective function. In many real-world applications, however, several objective functions have to be considered simultaneously. Recently, an algorithm for dealing with several objective functions in the i.i.d. case has been presented. 
In this paper, we extend the multi-objective framework to the case of stationary and ergodic processes, thus  allowing dependencies among observations. We first identify an asymptomatic lower bound for any prediction strategy and then present an algorithm whose predictions achieve the optimal solution while  fulfilling  any
continuous and convex constraining criterion.
\end{abstract} 
\section{Introduction}
In the traditional online learning setting, and in particular in sequential prediction under uncertainty, the learner is evaluated by a single loss function that is not completely known at each iteration  \cite{CesaL2006}.  When dealing with multiple objectives, since it is impossible to simultaneously minimize all of the objectives,  one objective is chosen  as the main  function to minimize,  leaving  the others  to be bound by pre-defined thresholds. Methods for dealing with one objective function can be transformed to deal with several objective functions by giving each objective a pre-defined weight. The difficulty, however, lies in assigning an appropriate weight to each objective in order to keep the objectives below a given threshold. This approach is very problematic in real world applications, where the player is required
to to satisfy certain constraints. For example, in online portfolio selection \cite{LiH2014,BorodinE2005}, 
the player may want to maximize wealth while keeping the risk (i.e., variance) contained below a certain threshold.  Another example is the Neyman-Pearson (NP) classification paradigm (see, e.g., \cite{RigolletT2011}) (which extends the objective in classical binary classification) where the goal is  to learn a classifier 
achieving low type II error whose type I error is kept  below a given threshold.

Recently, \cite{MahdaviJY2012} presented an algorithm for dealing with the case of one main objective and fully-known constraints. In a subsequent work,   \cite{MahdaviYJ2013}   proposed a framework for dealing with multiple objectives in the stochastic i.i.d. case, where the learner does not have full information about the objective functions. They proved that if there exists a solution that minimizes the main objective function while keeping the other objectives below given thresholds, then their algorithm will converge to the optimal solution. 

In this work, we study online prediction with multiple objectives but now 
consider the challenging general case where the unknown underlying process
is  stationary and ergodic, thus  allowing observations to depend on each other arbitrarily. 
We consider a non-parametric approach, which has been applied successfully in various application domains.  For example, in online portfolio selection, \cite{Gyorfi2007,GyorfiLU2006}, \cite{GyorfiS2003}, and \cite{LiHG2011} proposed non-parametric online strategies  that guarantee, under mild conditions, the best possible outcome. Another interesting example in this regard is the work on time-series   prediction by \cite{BiauKLG2010}, \cite{GyorfiL2005}, and \cite{BiauP2011}. A common theme to all these results is that the asymptotically optimal strategies are constructed by combining the predictions of many simple experts. 
The algorithm presented in this paper utilizes as a sub-routine the Weak Aggregating Algorithm of \cite{Vovk2007}, and \cite{KalnishkanV2005} to handle multiple 
objectives. While we discuss here the case of only two objective functions, 
our theorems can be  extended easily to any fixed number of  functions. 
\paragraph{Outline}The paper is organized as follows: In Section~\ref{sec:formulation}, we  define the multi-objective optimization framework under a jointly stationary and ergodic process. In Section~\ref{sec:Optimal},  we  identify an asymptotic lower-bound for any prediction strategy. In Section~\ref{sec:Algorithm}, we  present  Algorithm~\ref{alg:main}, which asymptotically achieves an optimal feasible solution. 

\section{Problem Formulation}
\label{sec:formulation}
 We consider the following prediction game.
Let $\cX \eqdef [-D,D]^d \subset \nr^d$ be a compact \emph{ observation space} where $D>0$. 
At each round,
 $n = 1, 2, \ldots$, the player is required to make a prediction $y_n\in \choS $, where $\choS \subset \nr ^m$ is a compact and convex set,  based on past observations,
 $X^{n-1}_1 \eqdef (x_1, \ldots , x_{n-1})$ and, $x_i \in \cX$ ($X_1^0$ is the empty observation).   After making the prediction $y_n$, the observation $x_n$ is revealed and the player suffers  two losses, $\ml(y_n,x_n)$  and  $\cl(y_n,x_n)$, where $\ml$ and $\cl$ are  real-valued continuous functions and  convex w.r.t. their first argument. 
 We view  the player's prediction strategy as a sequence 
 $\Strategy \eqdef \{\strategy_n\}^\infty_{n=1}$ of forecasting functions 
 $\strategy_n : \cX^{(n-1)} \rightarrow \cY$;  that is,
 the player's prediction  at round $n$ is given by $\strategy_n(X^{n-1}_1 )$. Throughout the paper we assume that $x_1, x_2, \ldots$ are realizations of random variables $X_1, X_2, \ldots$ such that the stochastic process $(X_n)^\infty_{-\infty}$  is jointly stationary and ergodic and $\np(\obsF_i \in \obsS)=1$.   
The player's goal is to play the game with a strategy that minimizes the average $\ml$-loss,  $\frac{1}{N}\sum_{i=1}^{N}\ml(\strategy(X_1^{i-1}),x_i)$,
while keeping the average $\cl$-loss  $\frac{1}{N}\sum_{i=1}^{N}\cl(\strategy(X_1^{i-1}),x_i)$  bounded below a prescribed threshold $\gamma$. Formally, we define the following:
\begin{defn} [$\gamma$-boundedness]
A prediction strategy $\Strategy$ will be called $\gamma$-bounded if 
\begin{gather*}
\limsup_{N\rightarrow\infty}\left(\frac{1}{N}\sum_{i=1}^{N} \cl(\strategy_{i}(X_1^{i-1}),X_{i})\right)\leq \gamma
\end{gather*}
almost surely. The set of all  $\gamma$-bounded strategies will be denoted $\Strategy_\gamma$.
\end{defn} 
\begin{defn}  [$\gamma$-feasible process]
\label{assu:Feasassumption}
We say that the stationary and ergodic process $\{X_i \}_{-\infty}^\infty $ is $\gamma$-feasible w.r.t. the functions $\ml$ and $\cl$, if for  $\np_\infty$, the regular conditional probability distribution of $\obsF_0$  given $\nf_{\infty}$ (the $\sigma$-algebra generated by the infinite past $\obs_{-1},\obs_{-2},\ldots$), and for a threshold $\gamma>0$, if there exists some $\cho' \in \choS$ such that $\nE_{\np_\infty}\left[\cl(\cho',\obs_0)\right] < \gamma$.
\end{defn}
If $\gamma$-feasibility holds, then we will denote by $\xoptim$ ($\xoptim$ is not necessarily unique)  the  solution to the following minimization problem:
\begin{equation}
\label{minprob}
\begin{aligned}
& \underset{\cho \in \choS}{\text{minimize}}
& & \nE_{\np_\infty}\left[\ml(\cho,\obs_0)\right] \\
& \text{subject to}
& & \nE_{\np_\infty}\left[\cl(\cho,\obs_0)\right] \leq \gamma, 
\end{aligned}
\end{equation} ~(\ref{minprob})  
and we define the \emph{$\gamma$-feasible optimal value} as
\begin{gather*}
\optimal  = \nE\left[ \nE_{\np_\infty}\left[\ml(\xoptim,\obs_0)\right]\right] \phantom{a} a.s. 
\end{gather*}
Note  that  problem~(\ref{minprob}) is a convex minimization problem over $\choS$, which 
in turn is a compact and convex subset of $\nr^{m}$. Therefore, the problem is equivalent to finding  the saddle point of the Lagrangian function \cite{BenN2012}, namely,
\begin{gather*}
\min_{\cho \in \choS}\max_{\lambda \in \nr^+}\mathcal{L}(\cho,\lambda),
\end{gather*}
where the Lagrangian is
\begin{gather*}
\mathcal{L}(\cho,\lambda)\eqdef  \left(\nE_{\np_\infty}\left[\ml (\cho,\obs_0)\right]+\lambda\left(E_{\np_\infty}\left[ \cl (\cho,\obs_0)\right]-\gamma\right)\right).
\end{gather*}
We denote the  optimal dual by $\lamoptim$ and assume that $\lamoptim$ is unique. Moreover, we set a constant \footnote{This can be done, for example, by imposing some regularity conditions on the constraint function (see, e.g., \cite{MahdaviYJ2013}).} $\lambda_{\max}$ such that $\lambda_{\max}>\lamoptim$, and set $\Lambda \eqdef [0,\lambda_{\max}]$.
We also define the \emph{instantaneous Lagrangian function}   as 
\begin{equation}
\label{eq:l_loss}
l(\cho,\lambda,\obsr) \eqdef \ml(\cho,\obsr)+\lambda\left(\cl(\cho,\obsr)-\gamma\right).
\end{equation}

In Brief, we are seeking  a strategy $\strategy \in \Strategy_\gamma$ that is as good as any other $\gamma$-bounded strategy, in terms of the average $\ml$-loss, when the underlying process is $\gamma$-feasible. 
Such a strategy will be called  $\gamma$-universal. 
  
\section{Optimallity of $\optimal$}
\label{sec:Optimal}

%
In this section, we  prove that the average $\ml$-loss of any $\gamma$-bounded prediction strategy cannot be smaller than $\optimal$, the $\gamma$-feasible optimal value. This result is a generalization of the well-known result of  \cite{Algoet1994}  regarding the best possible outcome under a single objective. Before stating and proving this optimallity result, we  state one known lemma and state and prove two lemmas that will be used repeatedly in this paper.   The first lemma is known as Breiman's generalized ergodic theorem.  The second and the third lemmas  concern the continuity of the saddle point w.r.t. the probability distribution. 
\begin{lemma}[Ergodicity, \cite{Breiman1957}] 
\label{lem:Ergodic}
Let $\mathbf{X} = \{\obsF_i\}_{-\infty}^{\infty}$ be a stationary and ergodic process.
For each positive integer $i$, let $T_i$ denote the operator that shifts any sequence
 by $i$ places to the left. Let $f_1, f_2,\ldots$ be a sequence of real-valued functions such that $\lim_{n \rightarrow \infty} f_n(\mathbf{X}) = f (\mathbf{X})$ almost surely, for some function   $f$. Assume that $\nE  \sup_n | f_n(\mathbf{X})| < \infty $. Then,
\begin{gather*}
\lim_{n \rightarrow \infty} \frac{1}{n}\sum_{i=1}^{n} f_i(T^i \mathbf{X}) = \nE f (\mathbf{X})  
\end{gather*}
almost surely.
\end{lemma}
\begin{lemma} [Continuity and Minimax]
\label{lem:L*}
Let $\choS,\Lambda,\obsS$ be compact real spaces. $l:~\choS~\times~\Lambda~\times~\obsS~\rightarrow\nr$
be a continuous function. Denote by $\np(\obsS)$ the space of all probability measures
on $\obsS$ (equipped with the topology of weak-convergence). Then the following function $L^{*} : \np(\obsS) \rightarrow \reals$  is continuous
\begin{gather}
\label{equ:l^*}
L^{*}(\mathbb{Q})=\inf_{\cho \in \choS }\sup_{\lambda\in\Lambda} \nE_{\mathbb{Q}}\left[l(\cho,\lambda,\obsr)\right].
\end{gather}
Moreover, for any $\mathbb{Q}\in\np(\obsS)$,
\begin{gather*}
\inf_{\cho \in \choS}\sup_{\lambda\in\Lambda} \nE_{\mathbb{Q}}\left[l(\cho,\lambda,\obsr )\right] = \sup_{\lambda\in\Lambda} \inf_{\cho \in \choS} \nE_{\mathbb{Q}}\left[l(\cho,\lambda,\obsr)\right]. 
\end{gather*}
\end{lemma}
\begin{proof}
 $\choS,\Lambda,\obsS$ are compact, implying that the function $l\left(\cho,\lambda,\obsr \right)$ is  bounded. Therefore, the function $L:~\choS~\times~\Lambda~\times~\np(\obsS)\rightarrow\nr$, defined as 
\begin{gather}
\label{eq:L}
L\left(\cho,\lambda,\mathbb{Q}\right) = \nE_\mathbb{Q}\left[ l\left(\cho,\lambda,\obsr \right) \right],
\end{gather}
is continuous. By applying Proposition~$7.32$ from \cite{BertsekasS1978}, we have that 
$ \sup_{\lambda\in\Lambda} \nE_{\mathbb{Q}}\left[l(\cho,\lambda,\obs)\right]$ is continuous in $\mathbb{Q}\times \choS$. Again applying  the same proposition, we  get the desired result. The last part of the lemma  follows directly from Fan's minimax theorem \cite{Fan1953}.
\end{proof}
\begin{lemma} [Continuity of the optimal selection]
\label{lem:contin}
Let $\choS ,\Lambda,\obsS$ be compact real spaces, and let $L$ be as defined in Equation~(\ref{eq:L}).  Then, there exist two measurable selection functions $h^\obs$,$h^\lambda$ such that
\begin{gather*}
h^\cho(\mathbb{Q}) \in\arg\min_{\cho \in \choS} \left(\max_{\lambda\in\Lambda}L(\cho,\lambda,\mathbb{Q})  \right),
\\
h^\lambda(\mathbb{Q}) \in\arg\max_{\lambda\in\Lambda} \left(\min_{\cho \in \choS}L(\cho,\lambda,\mathbb{Q})  \right)
\end{gather*}  
for any $\mathbb{Q} \in \np(\obsS)$. Moreover, let  $L^{*}$ be as defined in Equation~(\ref{equ:l^*}). Then, the set
\begin{gather*}
Gr(L^{*})\eqdef 
\{(u^{*},v^{*},\nq)\mid u^{*} \in h^\cho(\mathbb{Q})  , v^{*} \in h^\lambda(\mathbb{Q}), \nq \in \np(\obsS) \},
\end{gather*}
 is closed in $\choS \times \Lambda \times \np(\obsS)$.

 \end{lemma}
 
%
\begin{proof}
The first part of the proof follows immediately from  the minimax measurable theorem 
of \cite{Nowak1985} due to the compactness of $\choS,\Lambda,\obsS$ and the properties of the loss function $L$.
The proof of the second part is similar to the one presented in Theorem~$3$ of \cite{Algoet1988}. 
In order to show that $Gr(L^{*})$ is closed, it is enough to show that if (i) $\nq_{n}\rightarrow \nq_{\infty}$  in $\np(\obsS)$; (ii) $u_{n}\rightarrow u_{\infty}$ in $\choS$; (iii) $v_{n}\rightarrow v_{\infty}$ in $\Lambda$ and (iv) $u_n  \in h^\cho(\mathbb{Q}_n)  , v_n  \in h^\lambda(\mathbb{Q}_n)$ for
all $n$, then, 
\begin{gather*}
u_\infty  \in h^\cho(\nq_{\infty})  , v_\infty \in h^\lambda(\nq_{\infty}).
\end{gather*}
The function $L(\cho,\lambda,\nq)$, as defined in Equation~(\ref{eq:L}), 
 is continuous. Therefore,
 \begin{gather*}
\lim_{n\rightarrow \infty}L(u_{n} ,v_{n} ,\nq_{n})
=L(u_{\infty},v_{\infty},\nq_{\infty}).
\end{gather*}

It remains to show that  $ u_\infty \in h^\cho(\nq_{\infty})$  and  $v_\infty \in h^\lambda(\nq_{\infty})$.
From the optimality of $u_n$ and $v_n$, we obtain
\begin{gather}
L(u_{\infty},v_{\infty},\nq_{\infty})=\lim_{n\rightarrow \infty}L(u_{n} ,v_{n} ,\nq_{n}) 
=\lim_{n\rightarrow \infty}L^{*}(\nq_{n}) \label{p1}.
\end{gather}
Finally, from the continuity of $L^{*}$ (Lemma~\ref{lem:L*}), we get
\begin{gather*}
(\ref{p1}) =L^{*}(\lim_{n\rightarrow \infty}\nq_{n})=L^{*}(\nq_{\infty}),
\end{gather*}
which gives the desired result.
\end{proof}
 
\begin{corr}
\label{corr:con}
Under the conditions of Lemma 3. Define $L_n(\cho,\lambda,\nq) = L(\cho,\lambda,\nq)+\frac{||\cho||^2-||\lambda||^2}{n}$ and denote  $  h_{L_n}^\cho(\mathbb{Q}_n), h_{L_n}^\lambda(\mathbb{Q}_n) $ to be the measurable selection functions of  $L_n$. If  $\nq_n \rightarrow \nq_\infty $ weakly in $\np(\obsS)$ and $ u_n \in h_{L_n}^\cho(\mathbb{Q}_n), v_n\in h_{L_n}^\lambda(\mathbb{Q}_n) $, then 
\begin{gather*}
L_n(u_n,v_n,\nq_n)\rightarrow L(u_\infty,v_\infty,\nq_\infty) 
\end{gather*}
almost surely for  $ u_\infty \in h^\cho(\nq_\infty)$ and $v_\infty \in h^\lambda(\nq_\infty) $.
\end{corr}
\begin{proof}
Denote $\hat{u}_n\in h^\cho(\nq_\infty)$ and $\hat{v}_n\in h^\lambda(\nq_\infty) $
\begin{gather}
|L_n(u_n,v_n,\nq_n)-L(u_\infty,v_\infty,\nq_\infty)| \nonumber\\
 \leq |L_n(u_n,v_n,\nq_n)-L(\hat{u}_n,\hat{v}_n,\nq_n)| 
 +|L(\hat{u}_n,\hat{v}_n,\nq_n)-L(u_\infty,v_\infty,\nq_\infty)|. \label{app:e1}
\end{gather}
Note that for every $n$ and for constant $E>0$,
\begin{gather*}
\min_{\cho \in \choS}\max_{\lambda \in \Lambda} L(\cho,\lambda , \nq) - \frac{||\lambda_{\max}||^2}{n} \leq \min_{\cho \in \choS}\max_{\lambda \in \Lambda} L_{n}(\cho,\lambda , \nq) \\ 
= \min_{\cho \in \choS}\max_{\lambda \in \Lambda} \left( \nE_\nq\left[l(\cho,\lambda , \obs)\right] + \frac{ ||\cho||^2- ||\lambda||^2}{n} \right) \\
 \leq \min_{\cho \in \choS}\max_{\lambda \in \Lambda} L(\cho,\lambda , \nq)  + \frac{E}{n}.
\end{gather*}
Thus, for some constant $C$, $|L_n(u_n,v_n,\nq_n)-L(u_\infty,v_\infty,\nq_\infty)|<\frac{C}{n}$ 
and from Lemma 3, the last summand also converges to $0$ as $n$ approaches $\infty$,  we get the desired result, and clearly, if $h^\cho(\nq_\infty)$ and $ h^\lambda(\nq_\infty) $ are singletons, then, the only accumulation point of $\{(v_n,u_n)\}_{n=1}^\infty$ is $(v_\infty,u_\infty)$.
\end{proof}

The importance of Lemma~\ref{lem:contin} stems from the fact that it proves the continuity properties of the multi-valued correspondences $\nq \rightarrow h^\cho(\nq)$ and $\nq \rightarrow h^\lambda(\nq)$. This leads to the knowledge  that  if for  the limiting distribution, $\nq_\infty$, the optimal set is  a singleton, then  $\nq \rightarrow h^\cho(\nq)$ and $\nq \rightarrow h^\lambda(\nq)$ are continuous in $\nq_\infty$.  
%
We are now ready to prove the optimality of $\optimal$.
 \begin{thm}[Optimality of $\optimal$]
 \label{lem:optimal}
Let $\{X_i\}_{-\infty}^\infty$ be a $\gamma$-feasible process. Then, for any  strategy $\Strategy \in  \Strategy_\gamma$,  the following holds a.s.
\begin{gather*} 
\liminf_{N \rightarrow \infty} \frac{1}{N}\sum_{i=1}^N \ml(\strategy(X_1^{i-1}),\obs_i)\geq \optimal.
\end{gather*}
 
 \end{thm}
 \begin{proof}
For any given strategy $\Strategy \in \Strategy_\gamma $, we will look at the following sequence:
\begin{gather}
 \frac{1}{N}\sum_{i=1}^N l(\strategy(X_1^{i-1}),\tilde{\lambda_i^*},\obs_i). \label{t1p1} 
\end{gather}
where $\tilde{\lambda_i^*}\in h^\lambda(\np_{X_i\mid \obsF_1^{i-1}})$  
Observe that 
\begin{gather*}
 (\ref{t1p1}) = \frac{1}{N}\sum_{i=1}^N \nE \left[l(\strategy(X_1^{i-1}), \tilde{\lambda_i^*},\obs_i)\mid X_1^{i-1} \right]  
- \frac{1}{N}  \sum_{i=1}^N ( l(\strategy(X_1^{i-1}),\tilde{\lambda_i^*},\obs_i) \\ -  \nE \left[l(\strategy(X_1^{i-1}),\tilde{\lambda_i^*},\obs)  \mid X_1^{i-1} \right]).
\end{gather*}
Since $A_i = l(\strategy(X_1^{i-1}),\tilde{\lambda_i^*},\obs_i)-\nE \left[l(\strategy(X_1^{i-1}),\tilde{\lambda_i^*},\obs_i)  \mid X_1^{i-1} \right]$ is a martingale difference sequence, the last summand converges to $0$ a.s., by the strong law of large numbers (see, e.g., \cite{Stout1974}). Therefore,  
\begin{gather}
\liminf_{N\rightarrow \infty} \frac{1}{N}\sum_{i=1}^N l(\strategy(X_1^{i-1}),\tilde{\lambda_i^*},\obs_i) = 
 \liminf_{N\rightarrow \infty} \frac{1}{N}\sum_{i=1}^N \nE \left[l(\strategy(X_1^{i-1}),\tilde{\lambda_i^*},\obs_i)\mid X_1^{i-1} \right] \nonumber\\
  \geq  \liminf_{N\rightarrow \infty} \frac{1}{N}\sum_{i=1}^N \min_{\cho \in \choS()}\nE \left[l(\cho,\tilde{\lambda_i^*},\obs_i)\mid X_1^{i-1} \right],\label{p2} 
\end{gather}
where the minimum is taken w.r.t. all the $\sigma(X_1^{i-1})$-measurable functions. Because the  process is stationary, we get for $\hat{\lambda_i^*}\in h^\lambda(\np_{X_0\mid \obsF_{1-i}^{-1}})$,
\begin{gather}
(\ref{p2}) =  \liminf_{N\rightarrow \infty} \frac{1}{N}\sum_{i=1}^N \min_{\cho \in \choS()}\nE \left[l(\cho,\hat{\lambda_i^*},\obs_0)\mid X_{1-i}^{-1} \right] \\ 
 =  \liminf_{N\rightarrow \infty} \frac{1}{N}\sum_{i=1}^N L^*(\np_{X_0\mid \obsF_{1-i}^{-1}}). \label{p3}
\end{gather}
Using Levy's zero-one law, $\np_{X_0\mid \obsF_{1-i}^{-1}} \rightarrow \np_\infty$ weakly  as $i$ approaches $\infty$ and from  Lemma~\ref{lem:L*} we know that $L^*$ is continuous. 
  Therefore, we can apply Lemma~\ref{lem:Ergodic} and get that a.s.
\begin{gather}
(\ref{p3}) = \nE\left[L^*(\np_\infty)\right] = \nE\left[\nE_{\np_\infty}\left[l\left(\xoptim,\lamoptim,\obs_0\right) \right]\right]   
= \nE\left[\mathcal{L}\left(\xoptim,\lamoptim,\obs_0 \right) \right]. \label{p4}
\end{gather}
Note also, that due to the complementary slackness condition of the optimal solution, i.e., $\lamoptim(\nE_{\np_\infty}\left[\cl(\xoptim,\obs_0)\right]-\gamma)=0$, we  get
\begin{gather*}
(\ref{p4}) =  \nE\left[ \nE_{\np_\infty} \left[\ml \left(\xoptim,\obs_0 \right) \right]\right] = \optimal.
\end{gather*}
From the uniqueness of $\lamoptim$, and using Lemma~\ref{lem:contin} $\hat{\lambda_i^*} \rightarrow \lamoptim$ as $i$ approaches $\infty$. Moreover, since $l$ is continuous on a compact set, $l$ is also uniformly continuous. Therefore, for any given $\epsilon>0$, there exists $\delta >0$, such that if  $|\lambda' - \lambda| <\delta$, then $$|l(\cho,\lambda',\obsr) - l(\cho,\lambda,\obsr)| <\epsilon$$ for any $\cho \in \choS$ and $\obsr \in \obsS$. Therefore, there exists $i_0$ such that if $i> i_0$ then $|l(\cho,\hat{\lambda_i^*},\obsr) - l(\cho,\lamoptim,\obsr)| <\epsilon $   for any $\cho \in \choS$ and $\obsr \in \obsS$. Thus,

\begin{gather*}
\liminf_{N \rightarrow \infty} \frac{1}{N}  \sum_{i=1}^N  l(\strategy(X_1^{i-1}), \lamoptim,\obs_i)  - \liminf_{N \rightarrow \infty} \frac{1}{N}  \sum_{i=1}^N  l(\strategy(X_1^{i-1}), \hat{\lambda_i^*},\obs_i)   \\
= \liminf_{N \rightarrow \infty} \frac{1}{N}  \sum_{i=1}^N  l(\strategy(X_1^{i-1}), \lamoptim,\obs_i)  + \limsup_{N \rightarrow \infty} \frac{1}{N}  \sum_{i=1}^N - l(\strategy(X_1^{i-1}), \hat{\lambda_i^*},\obs_i)   \\
 \geq \
 \liminf_{N \rightarrow \infty} \frac{1}{N}  \sum_{i=1}^N  l(\strategy(X_1^{i-1}), \hat{\lambda_i^*},\obs_i) -\frac{1}{N}  \sum_{i=1}^N  l(\strategy(X_1^{i-1}), \lamoptim,\obs_i)  \geq -\epsilon \phantom{a} a.s.,
\end{gather*}
and since $\epsilon$ is arbitrary,
\begin{gather*}
\liminf_{N \rightarrow \infty} \frac{1}{N}  \sum_{i=1}^N  l(\strategy(X_1^{i-1}), \lamoptim,\obs_i)   \geq  \liminf_{N \rightarrow \infty} \frac{1}{N}  \sum_{i=1}^N   l(\strategy(X_1^{i-1}), \hat{\lambda_i^*},\obs_i).
\end{gather*}
Therefore we can conclude that
\begin{gather*}
 \liminf_{N \rightarrow \infty} \frac{1}{N}  \sum_{i=1}^N  l(\strategy(X_1^{i-1}), \lamoptim,\obs_i)  \geq \optimal \phantom{a} a.s.
\end{gather*}
We  finish the proof by noticing that since $\Strategy \in \Strategy_\gamma$,  then by definition
 \begin{gather*} 
\limsup_{N \rightarrow \infty} \frac{1}{N}\sum_{i=1}^N \cl(\strategy(X_1^{i-1}),\obs_i)\leq\gamma\phantom{a} a.s. 
\end{gather*}
and since $\lamoptim$ is non negative, we will get the desired result.
 \end{proof}
The above lemma  also provides the motivation to find the saddle point of the Lagrangian $\mathcal{L}$. Therefore,  for the reminder of the paper we will use the loss function $l$  as defined  in Equation~\ref{eq:l_loss}.  
\section{Minimax Histogram Based Aggregation}
\label{sec:Algorithm}
\begin{algorithm}[tb!]
   \caption{Minimax Histogram Based Aggregation (MHA)}
   \label{alg:main}
\begin{algorithmic}
\STATE \textbf{Input:} Countable set of experts $\{H_{k,h}\}$ ,
$\cho_{0}\in \choS$ $\lambda_{0}\in \Lambda$, initial probability $\{\alpha_{k,h}\}$, 

\STATE \textbf{For $n=0$ to $\infty$}

\STATE \quad{}Play $\cho_{n},\lambda_{n}$. 

\STATE \quad{}Nature reveals $\obsr_{n}$

\STATE \quad{}Suffer loss $l(\cho_{n},\lambda_{n},\obsr_{n})$.
\STATE \quad{}Update the cumulative loss of the experts
\begin{gather*}
 l_{\cho,n}^{k,h}  \eqdef \sum_{i=0}^{n} l(\cho^{i}_{k,h},\lambda_{i},\obsr_i)  \phantom{aaaa}  l_{\lambda,n}^{k,h}  \eqdef \sum_{i=0}^{n} l(\cho_{i},\lambda^{i}_{k,h},\obsr_i)
\end{gather*} 

\STATE \quad{}Update experts' weights  
\begin{gather*}
w_{n}^{\cho,(k,h)} \eqdef \alpha_{k,h}\exp\left(-\frac{1}{\sqrt{n}}l_{\cho,n}^{k,h}\right) \\
p_{n+1}^{\cho,(k,h)} \eqdef \frac{w_{n+1}^{\cho,(k,h)}}{\sum_{h=1}^{\infty}\sum_{k=1}^{\infty}w_{n+1}^{\cho,(k,h)}}
\end{gather*}
\STATE \quad{}Update  experts' weights $w_{n+1}^{\lambda,(k,h)}$ 
\begin{gather*}
w_{n+1}^{\lambda,(k,h)}  \eqdef \alpha_{k,h}\exp\left(\frac{1}{\sqrt{n}}l_{\lambda,n}^{k,h}\right) \\
p_{n+1}^{\lambda,(k,h)}=\frac{w_{n+1}^{\lambda,(k,h)}}{\sum_{h=1}^{\infty}\sum_{k=1}^{\infty}w_{n+1}^{\lambda,(k,h)}}
\end{gather*}
\STATE\quad{}Choose $\cho_{n+1}$ and $\lambda_{n+1}$ as follows
\begin{gather*}
\cho_{n+1} = \sum_{k,h}p_{n+1}^{\cho,(k,h)} \cho^{n+1}_{k,h} \phantom{aa} \lambda_{n+1} = \sum_{k,h}p_{n+1}^{\lambda,(k,h)} \lambda^{n+1}_{k,h}
\end{gather*} 
\STATE \textbf{End For}
\end{algorithmic}
\end{algorithm}

We are now ready to  present our algorithm \emph{Minimax Histogram based Aggregation (MHA)} and prove that its predictions are as good as the best strategy.
 By Theorem~\ref{lem:optimal} we can restate our goal: find a prediction strategy $\Strategy \in \Strategy_\gamma$ such that for any $\gamma$-feasible  process $\{\obsF_i \}_{-\infty}^\infty$ the following holds:
\begin{gather*}
\lim_{N \rightarrow \infty} \frac{1}{N}\sum_{i=1}^N \ml(\strategy(\obsF_1^{i-1}) ,\obs_i  ) = \optimal \phantom{a} a.s.
\end{gather*} 
Such a strategy will be called \emph{$\gamma$-universal}. We do so by maintaining a  countable set of experts $\{H_{k,h}\}$, where an  expert $H_{k,l}$ will output a pair $(\cho^i_{k,l} ,\lambda^i_{k,l})\in  \choS \times \Lambda$  at round $i$. Our algorithm  outputs at   round $i$ a pair $(\cho_i ,\lambda_i)\in  \choS \times \Lambda$ where the sequence of predictions $\cho_1,\cho_2,\ldots$ tries to minimize the average loss $\frac{1}{N}\sum_{i=1}^N l(\cho,\lambda_i,\obsr_i)$ and the sequence of predictions $\lambda_1,\lambda_2,\ldots$ tries to maximize the average loss $\frac{1}{N}\sum_{i=1}^N l(\cho_i,\lambda,\obsr_i)$. Each of  $\cho_i$ and $\lambda_i$ is the aggregation of predictions $\cho^i_{k,l}$ and $\lambda^i_{k,l}$, $k,l=1,2,\ldots,$ respectively. In order to ensure that  the performance  of MHA will be as good as any other expert for both the $\cho$ and the $\lambda$ predictions, we apply  the Weak Aggregating Algorithm of \cite{Vovk2007}, and \cite{KalnishkanV2005}  twice simultaneously. 
In Theorem~\ref{thm:achive} we prove  that there exists  a countable set of experts  whose selection of points converges to the optimal solution. Then, in Theorem~\ref{thm:Main} we   prove that MHA applied on  the experts defined in Theorem~\ref{thm:achive} generates a sequence of predictions that is $\gamma$-bounded and as good as any other strategy w.r.t. any $\gamma$-feasible process.
\begin{thm}
\label{thm:achive}
Assume that $\{\obsF_i\}_{-\infty}^{\infty}$ is a $\gamma$-feasible process. Then, it is possible to construct  a countable set of experts $\{H_{k,h}\}$ for which
\begin{gather*}
\lim_{k \rightarrow \infty} \lim_{h \rightarrow \infty} \lim_{n \rightarrow \infty} \frac{1}{N} \sum_{i=1}^N l(\cho^i_{k,h},\lambda^i_{k,h}  ,\obs_i) = \optimal\phantom{a} a.s., 
\end{gather*}
where $(\cho^i_{k,h},\lambda^i_{k,h})$ are the predictions made by expert $H_{k,h} $ at round $i$.
\end{thm}
\begin{proof}
We  start by defining  a countable set of experts $\{H_{k,h}\}$  as follow: For $h = 1,2,\ldots$,
let $P_h = \{A_{h,j} \mid   j = 1, 2, . . . ,m_h \}$ be a sequence of finite partitions of $\obsS$ such that: (i) any cell of $P_{h+1}$ is a subset of a cell of $P_h$ for any $h$. Namely, $P_{h+1}$ is a refinement of $P_{h}$; (ii) for a set $A$, if $diam(A) = \sup_{x,y\in A} ||x-y||$ denotes the diameter of $A$, then for any sphere $B$ centered at the origin,
$$
\lim_{h \rightarrow \infty} \max_{j: A_{h,j} \cap B \neq \emptyset } diam(A_{h,j}) =  0.
$$
Define the corresponding quantizer $q_h(x) = j$, if $x \in A_{h,j}$. Thus, for any $n$ and $\obsF_1^n$, we define  $Q_h(\obsF_1^n)$ as the sequence $q_h(\obsr_1), \ldots ,q_h(\obsr_n)$.
For expert $H_{k,h}$, we define for $k>0$, a $k$-long string of positive integers, denoted by $\stri$,  the following set, 
\begin{gather*}
B^{\stri ,(1,n-1)}_{k,h}\eqdef \{\obsr_i \mid k<i<n,\phantom{a} Q_h(\obsF_{i-k}^{i-1})= \stri \}. 
\end{gather*}
We define also
\begin{gather*}
h_{k,h}^\cho (\obsF_{1}^{n-1},\stri) \eqdef   
\arg\min_{\cho \in \choS } \left(\max_{\lambda\in \Lambda} \frac{1}{|B^{\stri,(1,n-1)}_{k,h}|}\sum_{\obsr_i\in B^{\stri,(1,n-1)}_{k,h}}l_{k,l,n}(\cho,\lambda,\obsr_{i})\right)
\\
h_{k,h}^\lambda (\obsF_{1}^{n-1},\stri) \eqdef 
\arg\max_{\lambda\in \Lambda} \left( \min_{\cho \in \choS} \frac{1}{|B^{\stri,(1,n-1)}_{k,h}|}\sum_{\obsr_i\in B^{\stri,(1,n-1)}_{k,h}}l_{k,l,n}(\cho,\lambda,\obsr_{i})\right)
\end{gather*}
for 
\begin{gather*}
l_{k,h,n}(\cho,\lambda,\obsr) \eqdef 
 l(\cho,\lambda,\obsr) +\left(||\cho||^2-||\lambda||^2\right) \left(\frac{1}{n} +\frac{1}{h} +\frac{1}{k} \right)   
\end{gather*}
and we will set $h^\cho_{k,h} (\obsF_{1}^{n-1},\stri) = \cho_{0} $ and $h^\lambda_{k,h} (\obsF_{1}^{n-1},\stri) =\lambda_{0}$ for arbitrary $\left(\cho_{0},\lambda_{0}\right) \in \choS \times \Lambda$ if $B^{\stri,(1,n-1)}_{k,h}$ is empty.
Using the above, we  define the predictions of $H_{k,h}$ to be:
\begin{gather*}
H^\cho_{k,h} (\obsF_{1}^{n-1})=h^\cho_{k,h}(\obsF_{1}^{n-1},Q(\obsF_{n-k}^{n-1})),\ n=1,2,3.... \\
H^\lambda_{k,h} (\obsF_{1}^{n-1})=h^\lambda_{k,h}(\obsF_{1}^{n-1},Q(\obsF_{n-k}^{n-1})),\ n=1,2,3....
\end{gather*}

%
%
We will add two  experts: $H_{0,0}$ whose predictions are always $(\cho_0 ,\lambda_{\max})$ and $H_{-1,-1}$ whose predictions are always $(\cho_0 ,0)$.  

Fixing $k,h>0$ and $\stri$, we will 
define a (random) measure $\np_{j,\stri}^{(k.h)}$ that is the measure
concentrated on the set $B_{k,h}^{\stri,(0,1-j)}$, defined by
\begin{gather*}
\np_{j,{\stri}}^{(k,h)}(A)=\frac{\sum_{\obsF_i\in B^{\stri,(0,1-j)}_{k,h}}1_{A}(\obsF_{i})}{|B^{\stri,(0,1-j)}_{k,h}|}, 
\end{gather*}
where $1_{A}$ denotes the indicator function of the set $A\subset \obsS$. If the above set  $B^{\stri}_{k,h}$ is empty, then let $P_{j,\stri}^{(k,h)}(A)=\delta(\obsr')$ be the probability measure concentrated on arbitrary  vector $\obsr'\in \obsS$.

In other words, $\np_{j,\stri}^{(k.h)}(A)$ is the relative frequency of
the the vectors among $\obsF_{1-j+k},\ldots,\obsF_{\ensuremath{0}}$ that fall
in the set $A$.
Applying  Lemma~\ref{lem:Ergodic} twice, it is straightforward to prove that for all $\stri$, w.p. $1$
\begin{gather*}
\np_{j,\stri}^{(k,h)}\rightarrow\begin{cases}
\np_{\obsF_{0}\mid G_l(\obsF_{-k}^{-1})=\stri} & \np(G_l(\obsF_{-k}^{-1})=\stri)>0\\
\delta(\obsr') & \textit{otherwise}
\end{cases}
\end{gather*}
weakly as $j\rightarrow\infty$, where $\np_{\obsF_{0}\mid G_l(\obsF_{-k}^{-1})=\stri}$
denotes the distribution of the vector $\obsF_{0}$ conditioned on the
event $G_l(\obsF_{-k}^{-1})=\stri$. 
To see this, let $f$ be a bounded continuous function. Then,
\begin{gather*}
\int f(x)\np_{j,\stri}^{(k,h)}(dx) = \frac{\frac{1}{|1-j+k|}\sum_{\obsF_i\in B^{\stri,(0,1-j)}_{k,h}}f(\obsF_{i})}{\frac{1}{|1-j+k|}|B^{\stri,(0,1-j)}_{k,h}|} \\
\rightarrow \frac{\nE \left[ f(\obsF_{0})1_{G_l(\obsF_{-k}^{-1})=\stri}(\obsF_{0})\right]}{\np( G_l(\obsF_{-k}^{-1})=\stri)} 
= \nE \left[ f(\obsF_{0}) \mid G_l(\obsF_{-k}^{-1})=\stri\right],
\end{gather*}
and in case $\np(||\obsF_{-k}^{-1}-s||\leq c/l)=0$, then w.p. $1$, $\np_{j,\stri}^{(k,h)}$ is concentrated on $x'$ for all $j$.  
We will denote the limit distribution of $\np_{j,\stri}^{(k,h)}$
by $\np_{\stri}^{*(k,h)}$.

By definition, $\left(h^\cho_{k,h}(\obsF_{1-n}^{-1},\stri),h^\lambda_{k,h}(\obsF_{1-n}^{-1},\stri)\right)$ is the minimax of $l_{n,k,h}$ w.r.t. $\np_{j,\stri}^{(k,h)}$. 
The sequence of functions $l_{n,k,h}$ converges uniformly as $n$ approaches   $\infty$ to 
$$l_{k,h}(\cho,\lambda,\obsr) = l(\cho,\lambda,\obsr) +\left(||\cho||^2-||\lambda||^2 \right) \left( \frac{1}{h} +\frac{1}{k} \right).   
$$
Note also that for any fixed $\nq$, $L_{k,h} (\cho,\lambda,\nq)= \nE_{\nq}\left[ l_{k,h} (\cho,\lambda,\obs) \right]$ is strictly convex in $\cho$ and strictly concave in $\lambda$, and therefore, has a unique saddle-point (see, e.g., \cite{LouHXSJ2016}). Therefore, since $\stri$ is arbitrary, and following  a Corollary~\ref{corr:con} of Lemma~\ref{lem:contin},  we get that a.s. 
\begin{gather*}
\cho_{k,h}^{n} \rightarrow  \cho_{k,h}^{*} \phantom{aaa} \lambda_{k,h}^{n} \rightarrow \lambda_{k,h}^{*},
\end{gather*}
where $\left(\cho_{k,h}^{*},\lambda_{k,h}^{*}\right)$
is the minimax of $L_{k,h}$ w.r.t. $\np_{\obsF^{-1}_{-k}}^{*(k,h)}$.
Thus, we can apply Lemma~\ref{lem:Ergodic} and  conclude that as $N$ approaches  $\infty$, 
$$
\frac{1}{N}\sum_{i=1}^{N}l(\cho_{k,h}^{i},\lambda_{k,h}^{i},\obs_{i})\rightarrow \nE\left[l(\cho_{k,h}^{*},\lambda_{k,h}^{*},\obs_0)\right].
$$
a.s.. 
We now evaluate 
$$
\lim_{h\rightarrow\infty} \nE\left[l(\cho_{k,h}^{*},\lambda_{k,h}^{*},\obs_0)\right].
$$
Using the properties of the partition $P_h$ (see, e.g., \cite{GyorfiL2005,GyorfiS2003}), we get that 
\[
\np_{\obsF_{-k}^{-1}}^{*(k,h)} \rightarrow\np_{\left\{ \obsF_{0}\mid \obsF_{-k}^{-1}\right\}} 
\]
weakly as $h\rightarrow\infty$. Moreover, the sequence of functions $l_{k,h}$ converges uniformly as $h$ approaches $\infty$ 
$$l_{k}(\cho,\lambda,\obsr) = l(\cho,\lambda,\obsr) +  \frac{||\cho||^2 -||\lambda||^2 }{k}.    
$$
Note also, that for any fixed $\nq$, $L_{k} (\cho,\lambda,\nq)= \nE_{\nq}\left[ l_{k} (\cho,\lambda,\obs) \right]$ is strictly convex-concave, and therefore, has a unique saddle point.  
Accordingly, by applying  Corollary~\ref{corr:con} again, we get that a.s.
\begin{gather*}
\cho_{k,h}^{*} \rightarrow  \cho_{k}^{*} \phantom{aaaa} \lambda_{k,h}^{*} \rightarrow \lambda_{k}^{*},
\end{gather*}
where $\left(\cho_{k}^{*},\lambda_{k}^{*}\right)$ is the minimax of $L_{k}$ w.r.t. $\np_{\left\{ \obsF_{0}\mid \obsF_{-k}^{-1}\right\}} $. Therefore, as $h$ approaches $\infty$, 
\begin{gather*}
l(\cho_{k,h}^{*},\lambda_{k,h}^{*},\obs_0) \rightarrow l\left(\cho_{k}^{*},\lambda_{k}^{*},\obs_0 \right) 
\end{gather*}
a.s.. Thus, by  Lebesgue's dominated convergence,  
\begin{gather*}
\lim_{h\rightarrow\infty} \nE\left[l(\cho_{k,h}^{*},\lambda_{k,h}^{*},\obs_0)\right]
=\nE\left[l\left(\cho_{k}^{*},\lambda_{k}^{*},\obs_0 \right)\right].
\end{gather*}

Notice that for any  $\nq \in \np(\obsS)$, the  distance between  the saddle point of $L_{k}$ w.r.t. $\nq $  and the the saddle point of $L$ w.r.t. $\nq $ converges to $0$ as $k$ approaches  $\infty$. To see this, notice that
\begin{gather*}
\min_{\cho \in \choS}\max_{\lambda \in \Lambda} L(\cho,\lambda , \nq) - \frac{||\lambda_{\max}||^2}{k} \leq \min_{\cho \in \choS}\max_{\lambda \in \Lambda} L_{k}(\cho,\lambda , \nq) \\ 
 \leq \min_{\cho \in \choS}\max_{\lambda \in \Lambda} L(\cho,\lambda , \nq)  + \frac{E}{k}
\end{gather*}
for some constant $E$, since $\choS$ is bounded.
The last part in our proof will be to show that if $(\hat{\cho^*_k},\hat{\lambda^*_k})$ is the minimax of $L$ w.r.t. $\np_{\left\{ \obsF_{0}\mid \obsF_{-k}^{-1}\right\}}$, then as $k$ approaches   $\infty$, $\nE\left[l\left(\hat{\cho^*_k},\hat{\lambda^*_k},\obs_0\right)\right]$ will converge a.s. to  $\optimal$ and so   $\nE\left[l\left(\cho^*_k,\lambda^*_k,\obs_0\right)\right]$.


To show this, we will use the sub-martingale convergence theorem twice.
First, we define $Z_k$ as
$$
Z_{k} \eqdef \min_{\cho \in \choS()}  \nE\left[ \max_{\lambda \in \Lambda()} \nE\left[l\left(\cho,\lambda,\obs_0\right) \mid \obsF_{-\infty}^{-1} \right] \mid \obsF_{-k}^{-1}\right], 
$$
where the minimum is taken w.r.t. all $\sigma(\obsF_{-k}^{-1})$-measurable strategies and the maximum is taken w.r.t. all $\sigma(\obsF_{-\infty}^{-1})$-measurable strategies.
Notice that $Z_k$ is a super-martingale. We can see this by using the tower property of conditional expectations,
$$
\nE[Z_{k+1}\mid \obsF_{-k}^{-1}]=\nE\left[\nE\left[Z_{k+1}\mid \obsF_{-k-1}^{-1}\right]\mid \obsF_{-k}^{-1}\right]
$$
and since $Z_{k+1}$ is the optimal choice in $\choS$ w.r.t. to  $\obsF_{-k-1}^{-1}$,
\[
\leq\nE\left[\nE[Z_{k}\mid \obsF_{-k-1}^{-1}]\mid \obsF_{-k}^{-1}\right] = \nE[Z_{k} \mid \obsF_{-k}^{-1}]   =Z_k.
\]


Note also that $\nE[Z_k]$ is uniformly bounded. Therefore, we can apply the  super-martingale convergence theorem and get that $Z_k \rightarrow Z_{\infty}$ a.s., where,  
\[
Z_{\infty}= \nE\left[l(\cho_{\infty}^{*},\lambda_{\infty}^{*},\obs_0)\mid \obsF_{-\infty}^{-1}\right] =\optimal,
\]
 and by using Lebesgue's dominated convergence theorem, also $\nE[Z_k] \rightarrow \nE[Z_{\infty}] =\optimal$.
Using the same arguments, $Z'_{k}$, defined as 
$$
Z'_{k} \eqdef \max_{\lambda \in \Lambda()}  \nE\left[ \min_{\cho \in \choS()} \nE\left[l\left(\cho,\lambda,\obs_0\right) \mid \obsF_{-\infty}^{-1} \right] \mid \obsF_{-k}^{-1}\right], 
$$
where the maximum is taken w.r.t. all $\sigma(\obsF_{-k}^{-1})$-measurable strategies and the minimum is taken w.r.t. all $\sigma(\obsF_{-\infty}^{-1})$-measurable strategies, is a sub-martingale that also converges a.s.   to $Z_{\infty}$ and thus $\nE[Z'_k] \rightarrow \nE[Z_{\infty}] =\optimal$. 

We  conclude the proof by noticing that the following relation holds for any $k$,  
\begin{gather*}
\nE [Z'_{k}]=\nE\left[\max_{\lambda \in \Lambda()}  \nE\left[ \min_{\cho \in \choS()} \nE\left[l\left(\cho,\lambda,\obs_0\right) \mid \obsF_{-\infty}^{-1} \right] \mid \obsF_{-k}^{-1}\right]\right] \\
\leq \nE\left[\max_{\lambda \in \Lambda()} \nE\left[\nE\left[l\left(\hat{\cho_k^*},\lambda,\obs_0\right)\mid \obsF_{-\infty}^{-1}\right] \mid \obsF_{-k}^{-1} \right]\right] 
\\
=  \nE\left[\max_{\lambda \in \Lambda()} \nE\left[l\left(\hat{\cho_k^*} ,\lambda,\obs_0\right) \mid \obsF_{-k}^{-1} \right]\right] = 
\nE\left[l\left(\hat{\cho_{k}^{*}},\hat{\lambda_k^{*}},\obs_0\right)\right],  
\end{gather*}
and using similar arguments we can show that also
$$
 \nE\left[l\left(\hat{\cho_{k}^{*}},\hat{\lambda_k^{*}},\obs_0\right)\right] \leq \nE [Z_{k}],  
$$
 
and since both $\nE [Z_{k}]$ and $\nE [Z'_{k}]$ converge   to $\optimal$, we get the desired result.

\end{proof}

Before stating the main theorem regarding MHA, we now state and prove the following  lemma, which  is used in the proof of the main result regarding MHA.
\begin{lemma}
\label{lem:ineq}
Let $\{H_{k,h}\}$ be a countable set of experts as defined in the proof of Theorem~\ref{thm:achive}. Then, the following relation holds a.s.:
\begin{gather*}
  \inf_{k,h} \limsup_{n \rightarrow \infty} \frac{1}{N} \sum_{i=1}^{N} l\left(\cho^{i}_{k,h}, \lambda_i,\obs_i \right) \leq  \optimal \\
\leq \sup_{k,h} \liminf_{n \rightarrow \infty}  \frac{1}{N} \sum_{i=1}^{N}  l\left(\cho_i, \lambda^{i}_{k,h},\obs_i \right), 
\end{gather*}
where $(\cho_i,\lambda_i)$ are the predictions of MHA when applied on $\{H_{k,h}\}$.
 \end{lemma}
 
\begin{proof}

Set
\begin{gather*}
 f(\cho,\nq) \eqdef \max_{\lambda \in \Lambda} \nE_{\nq}\left[ l\left(\cho, \lambda ,\obs_0\right) \right]. 
\end{gather*}
We will start from the LHS,  
\begin{gather}
  \inf_{k,h} \limsup_{n \rightarrow \infty} \frac{1}{N} \sum_{i=1}^{N} l\left(\cho^{i}_{k,h}, \lambda_i,\obs_i \right), \label{l4e1}
\end{gather}
and similarly to Lemma~\ref{lem:optimal}, by using the strong law of large numbers we can write
\begin{gather}
(\ref{l4e1})= 
     \inf_{k,h} \limsup_{n \rightarrow \infty} \frac{1}{N} \sum_{i=1}^{N} \nE\left[l\left(\cho^{i}_{k,h}, \lambda_i,\obs_0 \right) \mid \obsF_{1-i}^{-1} \right] \nonumber \\
\leq   \inf_{k,h} \limsup_{n \rightarrow \infty} \frac{1}{N} \sum_{i=1}^{N}  f(\cho^{i}_{k,h},\np_{X_0\mid \obsF_{1-i}^{-1}})  \phantom{a} a.s. \label{l4e2}  
\end{gather}
For fixed $k,h>0$, from the proof of Theorem~(\ref{thm:achive}), $\cho^{i}_{k,h} \rightarrow \cho^{*}_{k,h}$ a.s. as $i$ approaches $\infty$, and from Levy's zero-one law also $\np_{X_0\mid \obsF_{1-i}^{-1}} \rightarrow \np_{\infty}$ weakly. From Lemma~\ref{lem:L*} we know that $f$ is continuous, therefore, 
 we can apply Lemma~\ref{lem:Ergodic} and get that
%
\begin{gather}
(\ref{l4e2}) =    \inf_{k,h}  \nE\left[ \nE\left[ f(\cho^{*}_{k,h},\np_{\infty})  \right] \right] 
 \leq   \lim_{k\rightarrow \infty} \lim_{l\rightarrow \infty} \nE\left[ f(\cho^{*}_{k,h},\np_{\infty})\right]. \label{l4e3}
\end{gather}

From the uniqueness of the saddle point and from the proof of Theorem~(\ref{thm:achive}), for fiked $k>0$,  $$  \lim_{h \rightarrow \infty}\cho^{*}_{k,h} \rightarrow \cho^{*}_{k} $$ a.s.. Thus, from the continuity of $f$ we get that
$$
\lim_{h \rightarrow \infty} f(\cho^{*}_{k,h},\np_{\infty}) \rightarrow f(\cho^{*}_{k},\np_{\infty})
$$
and again by Lebesgue's dominated convergence,
\begin{gather}
(\ref{l4e3})= \lim_{k \rightarrow \infty}\nE\left[ f(\cho^{*}_{k},\np_{\infty})  \right] = \lim_{k \rightarrow \infty} \nE\left[\max_{\lambda \in \Lambda} \nE_{\np_{\infty}}\left[ l\left(\cho^{*}_{k}, \lambda ,\obs_0\right) \right]\right]\label{l4e4*}.
\end{gather}
Now, from Theorem~\ref{thm:achive} we know that every accumulation point of the sequence $\{\cho^{*}_{k} \}$ is in the optimal set
$$
\arg\min_{\cho \in \choS} \left( \max_{\lambda \in \Lambda} \nE_{\np_{\infty}}\left[ l\left(\cho, \lambda ,\obs_0\right) \right] \right).
$$
Therefore a.s.
$$
\lim_{k \rightarrow \infty} \max_{\lambda \in \Lambda} \nE_{\np_{\infty}}\left[ l\left(\cho^{*}_{k}, \lambda ,\obs_0\right) \right] \rightarrow   \nE_{\np_{\infty}}\left[ l\left(\xoptim, \lamoptim ,\obs_0\right) \right],  
$$
and using Lebesgue's dominated convergence,
\begin{gather*}
(\ref{l4e4*}) = \nE \left[\nE_{\np_{\infty}}\left[ l\left(\xoptim, \lamoptim ,\obs_0\right) \right]\right] =\optimal.
\end{gather*}
Using similar arguments, we can show the second part of the lemma.

\end{proof}
We are now ready to state and prove the optimality of MHA.
\begin{thm}[Optimality of MHA]
\label{thm:Main} 
Let $(\cho_i,\lambda_i)$ be the predictions generated by MHA when applied on $\{H_{k,h}\}$  as defined in the proof of Theorem~\ref{thm:achive}. Then, for any $\gamma$-feasible process $\{\obsF_i\}_{-\infty}^{\infty}$: MHA is a $\gamma$-bounded and $\gamma$-universal strategy.
 \end{thm}
\begin{proof}
We  first show that 
\begin{gather}
\lim_{N \rightarrow \infty} \frac{1}{N}\sum_{i=1}^N  l(\cho_i,\lambda_i ,\obs_i  ) = \optimal \phantom{a} a.s. \label{eq:MHAoptim}
\end{gather}
Applying Lemma~5 in \cite{KalnishkanV2005}, we know that the $x$ updates guarantee  that for every expert $H_{k,h}$,

\begin{gather}
\frac{1}{N}\sum_{i=1}^{N}l(\cho_{i},\lambda_{i},\obsr_{i}) \leq  \frac{1}{N}\sum_{i=1}^{N}l(\cho_{k,h}^{i},\lambda_{i},\obsr_{i}) + \frac{C_{k,h}}{\sqrt{N}} \label{t3e1}\\
\frac{1}{N}\sum_{i=1}^{N}  l(\cho_{i},\lambda_{i},\obsr_{i}) \geq  \frac{1}{N}\sum_{i=1}^{N}  l(\cho_{i},\lambda_{k,h}^{i},\obsr_{i}) - \frac{C'_{k,h}}{\sqrt{N}}, \label{t3e2} 
\end{gather}

 where $C_{k,h},C'_{k,h}>0$ are some constants independent of $N$. In particular, using Equation~(\ref{t3e1}),
\begin{gather*}
\frac{1}{N}\sum_{i=1}^{N}l(\cho_{i},\lambda_{i},\obsr_{i}) \leq \inf_{k,h}\left( \frac{1}{N}\sum_{i=1}^{N}l(\cho_{k,h}^{i},\lambda_{i},\obsr_{i}) + \frac{C_{k,h}}{\sqrt{N}}\right). 
\end{gather*}
Therefore, we get
\begin{gather}
\limsup_{N\rightarrow \infty} \frac{1}{N}\sum_{i=1}^{N}l(\cho_{i},\lambda_{i},\obsr_{i}) \nonumber \\
 \leq \limsup_{N\rightarrow \infty}  \inf_{k,h}\left( \frac{1}{N}\sum_{i=1}^{N}l(\cho_{k,h}^{i},\lambda_{i},\obsr_{i}) + \frac{C_{k,h}}{\sqrt{N}}\right) \nonumber\\
 \leq   \inf_{k,h} \limsup_{N\rightarrow \infty}  \left( \frac{1}{N}\sum_{i=1}^{N}l(\cho_{k,h}^{i},\lambda_{i},\obsr_{i}) + \frac{C_{k,h}}{\sqrt{N}}\right) \nonumber\\
 \leq   \inf_{k,h} \limsup_{N\rightarrow \infty}  \left( \frac{1}{N}\sum_{i=1}^{N}l(\cho_{k,h}^{i},\lambda_{i},\obsr_{i}) \right), \label{t3e3} 
\end{gather}
where in the last inequality we used the  fact that $\limsup$ is sub-additive. Using Lemma~(\ref{lem:ineq}), we get that
\begin{gather}
 (\ref{t3e3}) \leq   \optimal \nonumber \\
\leq \sup_{k,h} \liminf_{n \rightarrow \infty}  \frac{1}{N} \sum_{i=1}^{N}  l\left(\cho_i, \lambda_{k,h}^{i},\obs_i \right). \label{t3e4}  
\end{gather}
Using similar arguments and using Equation~(\ref{t3e2}) we can show  that
\begin{gather*}
(\ref{t3e4}) \leq \liminf_{N\rightarrow \infty} \frac{1}{N}\sum_{i=1}^{N}l(\cho_{i},\lambda_{i},\obsr_{i}).
\end{gather*}
Summarizing, we have
\begin{gather*}
\limsup_{N\rightarrow \infty} \frac{1}{N}\sum_{i=1}^{N}l(\cho_{i},\lambda_{i},\obsr_{i}) \leq
\optimal \leq \liminf_{N\rightarrow \infty} \frac{1}{N}\sum_{i=1}^{N}l(\cho_{i},\lambda_{i},\obsr_{i}).
\end{gather*}
Therefore, we can conclude that a.s.
$$
\lim_{N\rightarrow \infty}  \frac{1}{N}\sum_{i=1}^{N} l(\cho_{i},\lambda_{i},\obs_{i}) = \optimal.
$$
To show that MHA is indeed a $\gamma$-bounded strategy and to shorten the notation, we will denote $$g(\cho,\lambda,\obsr) \eqdef \lambda(\cl(\cho,\obsr)-\gamma).$$
First, from Equation~(\ref{t3e2}) applied on the expert $H_{0,0}$, we get that:
\begin{gather}
\limsup_{N \rightarrow \infty}  \frac{1}{N}\sum_{i=1}^{N}  g(\cho_i,\lambda_{\max},\obsr) \leq 
\limsup_{N \rightarrow \infty}  \frac{1}{N}\sum_{i=1}^{N} g(\cho_i,\lambda_i,\obsr).   \label{t3e5}
\end{gather} 
%

Moreover, since $l$ is  uniformly continuous, for any given $\epsilon>0$, there exists $\delta >0$, such that if  $|\lambda' - \lambda| <\delta$, then $$|l(\cho,\lambda',\obsr) - l(\cho,\lambda,\obsr)| <\epsilon$$ for any $\cho \in \choS$ and $\obsr \in \obsS$.
We also know that $$\lim_{k \rightarrow \infty } \lim_{h \rightarrow  \infty } \lim_{i \rightarrow  \infty } \lambda_{k,h}^i = \lamoptim.$$ Therefore, there exist $k_0 ,h_0 ,i_0$ such that $ |\lambda_{k_0 ,h_0}^{i} - \lamoptim| <\delta $ for any $i>i_0$.  
Since $\lim_{k \rightarrow \infty } \lambda_k^* = \lamoptim$ there exists $k_0$ such that $|\lambda_{k_0}^* -\lamoptim |<\frac{\delta}{3}$. Note that  $\lim_{h \rightarrow \infty } \lambda_{k_0,h}^* = \lambda_{k_0}^*$, so there exists $h_0$ such that $|\lambda_{k_0,h_0}^* - \lambda_{k_0}^*|<\frac{\delta}{3}$. Finally, since $\lim_{i \rightarrow \infty } \lambda_{k_0,l_0}^i = \lambda_{k_0,l_0}^*$, there exists $i_0$ such that if $i>i_0$, then $|\lambda_{k_0,l_0}^{i} - \lambda_{k_0,l_0}^*|<\frac{\delta}{3}$. Combining all the above, we get that for $k_0 ,h_0 ,i_0$ if $i>i_0$, then
\begin{gather*}
|\lambda_{k_0 ,h_0}^{i} - \lamoptim| < |\lambda_{k_0 ,h_0}^{i} - \lambda_{k_0 ,h_0}^{*}|+|\lambda_{k_0 ,h_0}^{i}  - \lambda_{k_0}^{*}| 
+ |\lambda_{k_0}^{*}- \lamoptim| < \delta.  
\end{gather*}
Therefore,

\begin{gather}
\limsup_{N \rightarrow \infty} \left(  \frac{1}{N}\sum_{i=1}^{N}l(\cho_{i} ,\lamoptim,\obsr_{i}) - \frac{1}{N}\sum_{i=1}^{N}  l(\cho_{i},\lambda_i,\obsr_{i}) \right) \leq \nonumber  \\ 
\limsup_{N \rightarrow \infty}  \left(   \frac{1}{N}\sum_{i=1}^{N}l(\cho_{i} ,\lamoptim,\obsr_{i}) -\frac{1}{N}\sum_{i=1}^{N} l(\cho_i,\lambda_{k_0,h_0}^{i},\obsr_{i})  \right) + \nonumber \\
\limsup_{N \rightarrow \infty} \left(  \frac{1}{N}\sum_{i=1}^{N} l(\cho_{i},\lambda_{k_0,h_0}^{i},\obsr_{i}) - \frac{1}{N}\sum_{i=1}^{N}   l(\cho_{i},\lambda_i,\obsr_{i}) \right) \label{t3e6}
\end{gather} 
From the uniform continuity we also learn   that the first summand is bounded above by $\epsilon$, and from Equation~(\ref{t3e2}), we get that the last summand is  bounded above by  $0$. Thus,
\begin{gather*}
 (\ref{t3e6}) \leq  \epsilon, 
\end{gather*}
and since  $\epsilon$ is arbitrary, we get that
\begin{gather*}
\limsup_{N \rightarrow \infty} \left(  \frac{1}{N}\sum_{i=1}^{N}l(\cho_{i} ,\lamoptim,\obsr_{i}) - \frac{1}{N}\sum_{i=1}^{N}  l(\cho_{i},\lambda_i,\obsr_{i}) \right) \leq 0.
\end{gather*} 
Thus, 
\begin{gather*}
\limsup_{N \rightarrow \infty}  \frac{1}{N}\sum_{i=1}^{N}l(\cho_{i} ,\lamoptim,\obs_{i})  \leq \optimal,
\end{gather*}
and from Theorem~1 we can conclude that
\begin{gather*}
\lim_{N \rightarrow \infty}  \frac{1}{N}\sum_{i=1}^{N}l(\cho_{i} ,\lamoptim,\obs_{i})  = \optimal.
\end{gather*}
Therefore, we can deduce that
\begin{gather*}
\limsup_{N \rightarrow \infty}  \frac{1}{N}\sum_{i=1}^{N}  g(\cho_i,\lambda_i,\obsr_i) - \limsup_{N \rightarrow \infty}  \frac{1}{N}\sum_{i=1}^{N}  g(\cho_i,\lamoptim,\obsr_i) \nonumber  = \\
\limsup_{N \rightarrow \infty}  \frac{1}{N}\sum_{i=1}^{N}  g(\cho_i,\lambda_i,\obsr_i) + \liminf_{N \rightarrow \infty}  \frac{1}{N}\sum_{i=1}^{N}  -g(\cho_i,\lamoptim,\obsr_i) \nonumber \\
\leq \limsup_{N \rightarrow \infty}  \frac{1}{N}\sum_{i=1}^{N}  g(\cho_i,\lambda_i,\obsr_i) - \frac{1}{N}\sum_{i=1}^{N}  g(\cho_i,\lamoptim,\obsr_i)  \\
= \limsup_{N \rightarrow \infty}  \frac{1}{N}\sum_{i=1}^{N}  l(\cho_i,\lambda_i,\obsr_i) - \frac{1}{N}\sum_{i=1}^{N}  l(\cho_i,\lamoptim,\obsr_i) = 0, 
\end{gather*}
which results in
\begin{gather*}
\limsup_{N \rightarrow \infty}  \frac{1}{N}\sum_{i=1}^{N}  g(\cho_i,\lambda_i,\obsr_i) \leq \limsup_{N \rightarrow \infty}  \frac{1}{N}\sum_{i=1}^{N}  g(\cho_i,\lamoptim,\obsr_i).   
\end{gather*}
Combining the above with Equation~(\ref{t3e5}), we get that
\begin{gather*}
\limsup_{N \rightarrow \infty}  \frac{1}{N}\sum_{i=1}^{N}  g(\cho_i,\lambda_{\max},\obsr_i) \\
 \leq \limsup_{N \rightarrow \infty}  \frac{1}{N}\sum_{i=1}^{N}  g(\cho_i,\lamoptim,\obsr_i).   
\end{gather*}
Since $ 0 \leq \lamoptim < \lambda_{\max}$, we get that MHA is  $\gamma$-bounded. This also implies that 
$$\limsup_{N\rightarrow \infty} \frac{1}{N}\sum_{i=1}^N \lambda_i(\cl(\cho_i,\obsr_i)-\gamma) \leq 0.$$ 
 Now, if we apply  Equation~(\ref{t3e2})  on the expert $H_{-1,-1}$, we get that
$$\liminf_{N\rightarrow \infty} \frac{1}{N}\sum_{i=1}^N \lambda_i(\cl(\cho_i,\obsr_i)-\gamma) \geq 0.$$
Thus, 
$$\lim_{N\rightarrow \infty} \frac{1}{N}\sum_{i=1}^N \lambda_i(\cl(\cho_i,\obsr_i)-\gamma) = 0,$$
and using Equation~(\ref{eq:MHAoptim}), we get that MHA is also $\gamma$-universal.
\end{proof}

\section{Concluding Remarks}
In this paper, we   introduced the Minimax Histogram Aggregation (MHA) algorithm for multiple-objective sequential prediction.  We considered the general setting where the unknown
underlying process is stationary and ergodic., and given that the underlying process is $\gamma$-feasible, we extended the well-known  result of \cite{Algoet1994} regarding the asymptotic lower bound of prediction with a single objective, to the case of multi-objectives.   We proved that  MHA  is  a $\gamma$-bounded strategy whose predictions also converge to the optimal solution in hindsight. 

In the proofs of the theorems and lemmas above, we used the fact that the initial   weights of the experts, $\alpha_{k,h}$, are strictly positive thus implying a countably infinite expert set. In practice, however, one cannot maintain an infinite set of experts. Therefore, it is customary to apply such algorithms with a finite number of experts  (see \cite{Gyorfi2007,GyorfiLU2006,GyorfiS2003,LiHG2011}). 
Despite the fact that  in the proof we assumed that the observation set $\obsS$ is known a priori,  the algorithm can also  be   applied in the case that $\obsS$ is unknown by applying the doubling trick. For a further discussion on this point, see \cite{GyorfiL2005}.     
In our proofs, we relied  on the compactness of the set $\obsS$. It will be interesting to see whether the universality of MHA can be sustained under unbounded processes as well. A very interesting open question would be to identify conditions allowing for finite sample bounds when predicting with multiple objectives. 

\bibliography{Bibliography}

\begin{thebibliography}{10}

\bibitem{Algoet1994}
P.H. Algoet.
\newblock The strong law of large numbers for sequential decisions under
  uncertainty.
\newblock {\em IEEE Transactions on Information Theory}, 40(3):609--633, 1994.

\bibitem{Algoet1988}
P.H. Algoet and T.M. Cover.
\newblock Asymptotic optimality and asymptotic equipartition properties of
  log-optimum investment.
\newblock {\em The Annals of Probability}, pages 876--898, 1988.

\bibitem{BenN2012}
A.~Ben-Tal and A.~Nemirovsky.
\newblock Optimization iii.
\newblock {\em Lecture Notes}, 2012.

\bibitem{BertsekasS1978}
D.~Bertsekas and S.~Shreve.
\newblock {\em Stochastic optimal control: The discrete time case}, volume~23.
\newblock Academic Press New York, 1978.

\bibitem{BiauKLG2010}
G.~Biau, K.~Bleakley, L.~Gy{\"o}rfi, and G.~Ottucs{\'a}k.
\newblock Nonparametric sequential prediction of time series.
\newblock {\em Journal of Nonparametric Statistics}, 22(3):297--317, 2010.

\bibitem{BiauP2011}
G.~Biau and B.~Patra.
\newblock Sequential quantile prediction of time series.
\newblock {\em IEEE Transactions on Information Theory}, 57(3):1664--1674,
  2011.

\bibitem{BorodinE2005}
A.~Borodin and R.~El-Yaniv.
\newblock {\em Online Computation and Competitive Analysis}.
\newblock Cambridge University Press, 2005.

\bibitem{Breiman1957}
L.~Breiman.
\newblock The individual ergodic theorem of information theory.
\newblock {\em The Annals of Mathematical Statistics}, 28(3):809--811, 1957.

\bibitem{CesaL2006}
N.~Cesa-Bianchi and G.~Lugosi.
\newblock {\em Prediction, Learning, and Games}.
\newblock Cambridge University Press, 2006.

\bibitem{Fan1953}
K.~Fan.
\newblock Minimax theorems.
\newblock {\em Proceedings of the National Academy of Sciences}, 39(1):42--47,
  1953.

\bibitem{GyorfiL2005}
L.~Gy{\"o}rfi and G.~Lugosi.
\newblock Strategies for sequential prediction of stationary time series.
\newblock In {\em Modeling uncertainty}, pages 225--248. Springer, 2005.

\bibitem{GyorfiLU2006}
L.~Gy{\"o}rfi, G.~Lugosi, and F.~Udina.
\newblock Nonparametric kernel-based sequential investment strategies.
\newblock {\em Mathematical Finance}, 16(2):337--357, 2006.

\bibitem{GyorfiS2003}
L.~Gy{\"o}rfi and D.~Sch{\"a}fer.
\newblock Nonparametric prediction.
\newblock {\em Advances in learning theory: methods, models and applications},
  339:354, 2003.

\bibitem{Gyorfi2007}
L.~Gy{\"o}rfi, A.~Urb{\'a}n, and I.~Vajda.
\newblock Kernel-based semi-log-optimal empirical portfolio selection
  strategies.
\newblock {\em International Journal of Theoretical and Applied Finance},
  10(03):505--516, 2007.

\bibitem{KalnishkanV2005}
Y.~Kalnishkan and M.~Vyugin.
\newblock The weak aggregating algorithm and weak mixability.
\newblock In {\em International Conference on Computational Learning Theory},
  pages 188--203. Springer, 2005.

\bibitem{LiH2014}
B.~Li and S.C.H. Hoi.
\newblock Online portfolio selection: A survey.
\newblock {\em ACM Computing Surveys (CSUR)}, 46(3):35, 2014.

\bibitem{LiHG2011}
B.~Li, S.C.H Hoi, and V.~Gopalkrishnan.
\newblock Corn: Correlation-driven nonparametric learning approach for
  portfolio selection.
\newblock {\em ACM Transactions on Intelligent Systems and Technology (TIST)},
  2(3):21, 2011.

\bibitem{LouHXSJ2016}
Y.~Lou, Y.~Hong, L.~Xie, G.~Shi, and K.~Johansson.
\newblock Nash equilibrium computation in subnetwork zero-sum games with
  switching communications.
\newblock {\em IEEE Transactions on Automatic Control}, 61(10):2920--2935,
  2016.

\bibitem{MahdaviJY2012}
M.~Mahdavi, R.~Jin, and T.~Yang.
\newblock Trading regret for efficiency: online convex optimization with long
  term constraints.
\newblock {\em Journal of Machine Learning Research}, 13(Sep):2503--2528, 2012.

\bibitem{MahdaviYJ2013}
M.~Mahdavi, T.~Yang, and R.~Jin.
\newblock Stochastic convex optimization with multiple objectives.
\newblock In {\em Advances in Neural Information Processing Systems}, pages
  1115--1123, 2013.

\bibitem{Nowak1985}
A.~Nowak.
\newblock Measurable selection theorems for minimax stochastic optimization
  problems.
\newblock {\em SIAM Journal on Control and Optimization}, 23(3):466--476, 1985.

\bibitem{RigolletT2011}
P.~Rigollet and X.~Tong.
\newblock Neyman-pearson classification, convexity and stochastic constraints.
\newblock {\em Journal of Machine Learning Research}, 12(Oct):2831--2855, 2011.

\bibitem{Stout1974}
W.~Stout.
\newblock Almost sure convergence, vol. 24 of probability and mathematical
  statistics, 1974.

\bibitem{Vovk2007}
V.~Vovk.
\newblock Competing with stationary prediction strategies.
\newblock In {\em International Conference on Computational Learning Theory},
  pages 439--453. Springer, 2007.

\end{thebibliography}
\bibliographystyle{plain}
\end{document}